\documentclass[letterpaper,10pt,conference]{ieeeconf}

\IEEEoverridecommandlockouts
\usepackage{amsmath,amssymb}
\usepackage{graphicx}
\usepackage{booktabs}
\usepackage{tabularx}
\usepackage{cite}
\usepackage{tikz}
\usepackage{pgfplots}
\pgfplotsset{compat=1.18}
\usetikzlibrary{arrows.meta,calc,positioning,patterns,
                decorations.pathreplacing,backgrounds,fit}
\usepgfplotslibrary{groupplots,fillbetween}
\definecolor{bspTruth}{HTML}{000000}
\definecolor{bspEst}{HTML}{7F7F7F}
\definecolor{bspPlan}{HTML}{0072B2}
\definecolor{bspAcc}{HTML}{009E73}
\definecolor{bspRej}{HTML}{E69F00}
\definecolor{bspInfl}{HTML}{D55E00}
\definecolor{bspAlt}{HTML}{CC79A7}
\colorlet{bspDrop}{bspRej!14}

\providecommand{\bspdata}{figures/data}

\newif\ifbspplaceholder
\bspplaceholdertrue

\newcommand{\bspshade}[4]{%
  \pgfonlayer{axis background}%
    \fill[bspDrop] (axis cs:#1,#3) rectangle (axis cs:#2,#4);%
  \endpgfonlayer}

\tikzset{
  bsp arrow/.style={-{Stealth[length=4pt,width=3pt]}, line width=0.6pt},
  bsp darrow/.style={{Stealth[length=3pt,width=2.4pt]}-{Stealth[length=3pt,width=2.4pt]}, line width=0.5pt},
  bsp block/.style={draw, rounded corners=1.5pt, line width=0.6pt,
                    align=center, font=\scriptsize, inner sep=2.5pt},
  bsp label/.style={font=\scriptsize, inner sep=1pt},
  bsp panel/.style={font=\footnotesize, inner sep=0pt, anchor=north west},
}

\pgfplotsset{
  bsp base/.style={
    set layers,
    scale only axis,
    axis line style={line width=0.5pt},
    tick style={line width=0.4pt, black},
    major tick length=2pt,
    minor tick length=1pt,
    tick align=outside,
    tick pos=left,
    tick label style={font=\scriptsize},
    label style={font=\scriptsize},
    xlabel style={yshift=2pt},
    ylabel style={yshift=-2pt},
    title style={font=\scriptsize, yshift=-3pt},
    legend style={font=\scriptsize, draw=none, fill=white, fill opacity=0.85,
                  text opacity=1, inner sep=1pt, row sep=-1pt,
                  /tikz/every even column/.append style={column sep=4pt}},
    legend cell align=left,
    every axis plot/.append style={line width=0.8pt},
    clip mode=individual,
    enlargelimits=false,
  },
  bsp zero/.style={black!35, line width=0.4pt, forget plot},
  bsp truth/.style={bspTruth, line width=0.7pt},
  bsp plan/.style={bspPlan, line width=0.9pt},
  bsp planband/.style={draw=none, fill=bspPlan, fill opacity=0.22, forget plot},
  bsp estband/.style={draw=none, fill=bspEst, fill opacity=0.20, forget plot},
  bsp estedge/.style={bspEst, line width=0.5pt, dashed, forget plot},
  bsp infl/.style={bspInfl, line width=0.9pt, densely dashed},
  /pgf/declare function={
    gaussv(\x,\mu,\v) = exp(-((\x-\mu)^2)/(2*\v))/sqrt(2*pi*\v);
  },
}

\makeatletter
\newcommand{\figstylecaption}{%
  \long\def\@makecaption##1##2{%
    \vskip0.5\baselineskip
    \setbox\@tempboxa\hbox{\normalfont\footnotesize ##1.\nobreakspace\nobreakspace ##2}%
    \ifdim\wd\@tempboxa>\hsize
      \parbox[t]{\hsize}{\normalfont\footnotesize\noindent ##1.\nobreakspace\nobreakspace ##2}%
    \else
      \hbox to\hsize{\hfil\box\@tempboxa\hfil}%
    \fi}}
\makeatother

\newcommand{\E}{\mathbb{E}}
\newcommand{\Prob}{\mathbb{P}}
\newcommand{\Cov}{\operatorname{Cov}}
\newcommand{\tr}{\operatorname{tr}}

\newtheorem{proposition}{Proposition}
\newtheorem{assumption}{Assumption}

\title{%
Belief-Space Planning with Planner-Conditioned Estimator Error under Intermittent Observations
}

\author{Peter M. Donley and Joseph L. Moore
\thanks{Johns Hopkins University Whiting School of Engineering
        {\tt\small \{pdonley1, jlmoore\}@jhu.edu}}%
}

\begin{document}

\maketitle


\begin{abstract}

Belief-space planners with separately designed or off-the-shelf estimators may have access to state-relevant information the estimator does not observe. Consequently, even an estimator that minimizes mean-squared error under its own information can have a nonzero error mean when conditioned on planner information. When future corrections are intermittent and stochastic, differences between accepted and rejected error means introduce additional uncertainty terms to correction events that zero-mean assumptions ignore. In this paper, we propose a belief-space planning approach that predicts, propagates, and penalizes planner-conditioned estimator error along evaluated trajectories. A planner-conditioned estimator-error model combines affine error dynamics with a moment recursion that preserves stochastically-induced correction uncertainty. We describe a method by which to predict estimator error dynamics within specified operating regimes and incorporate predicted error moments into a task-weighted quadratic risk objective. We evaluate the approach in simulation for a tilt-rotor VTOL landing on a ship deck using receding-horizon planning. We find that conditioning on planner information reduced error-prediction loss for a command-blind EKF by 13.4\% relative to a planner-ignorant model, with strongly regime and horizon-dependent forecasting abilities. In a small set of 16 paired closed-loop trials, the planner lowered the median terminal task gauge from 1.60 to 0.99, and predicted estimator error beyond 2s within a factor of 1.3, versus a 5.7-fold underprediction by covariance-only planning.
\end{abstract}


\section{INTRODUCTION}
Planners and estimators are often designed separately, and as such, may observe different subsets of a system's information.  In particular, planners may have limited information about the internal state representations and update procedures of ``off-the-shelf" or otherwise black-box estimators, and an estimator may not receive direct information about the planner's commands or expected future states. If the planner has relevant information the estimator lacks, an estimator with zero-mean error under its own information may have a nonzero planner-conditioned error mean. Consider the task of landing an aerial robot on a ship-deck. Deck-relative observations become intermittent over the course of the task. During these periods of intermittency, estimator error accumulates due to a lack of posterior correction availability. However, the planner may be able to partially predict this error and incorporate its predictions into trajectory design. We study the interaction between this predictable error mean and the stochastic availability of future corrections.

Conventional covariance-only belief-space planning generally assumes zero-mean estimator error. However, if nonzero-mean estimator error exists, this can have a profound impact on planner performance, as the planner will be unable to adequately predict future expected estimation error, probability of correction acceptance, or post-correction uncertainty. Additionally, when correction changes the predicted conditional error mean, stochastic correction outcomes contribute a spread-of-the-means covariance term. The resulting first two moments generally cannot be reproduced by simply inflating measurement noise to account for probabilistic observation.
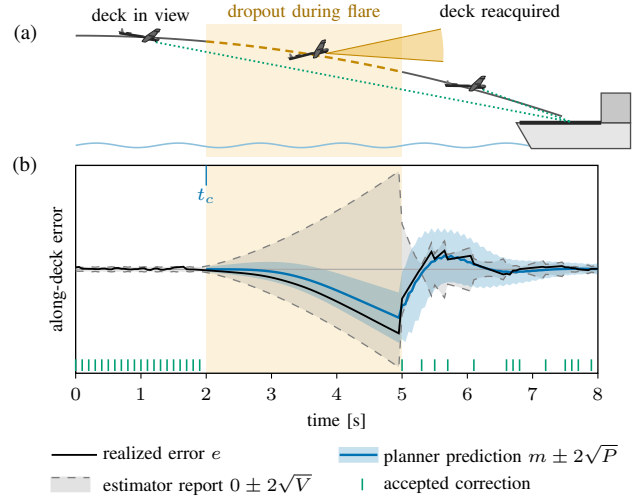
\begin{figure}[t]
    \centering
\tikzset{
vtol tilt/.store in=\vtoltilt, vtol tilt=90,  
vtol yaw/.store in=\vtolyaw,   vtol yaw=-18,  
vtol elev/.store in=\vtolelev, vtol elev=18,  
vtol/.pic={
  \pgfmathsetmacro\Xx{cos(\vtolyaw)}
  \pgfmathsetmacro\Xy{sin(\vtolyaw)*sin(\vtolelev)}
  \pgfmathsetmacro\Yx{-sin(\vtolyaw)}
  \pgfmathsetmacro\Yy{cos(\vtolyaw)*sin(\vtolelev)}
  \pgfmathsetmacro\Zy{cos(\vtolelev)}
  \pgfmathsetmacro\Ax{cos(\vtoltilt)*\Xx}
  \pgfmathsetmacro\Ay{cos(\vtoltilt)*\Xy+sin(\vtoltilt)*\Zy}
  \pgfmathsetmacro\Bx{-sin(\vtoltilt)*\Xx}
  \pgfmathsetmacro\By{-sin(\vtoltilt)*\Xy+cos(\vtoltilt)*\Zy}
  \begin{scope}[x={(\Xx cm,\Xy cm)}, y={(\Yx cm,\Yy cm)}, z={(0cm,\Zy cm)}]
    \fill[black!55] (-0.235,0,0.010) -- (-0.280,0.070,0.060) -- (-0.300,0.070,0.060) -- (-0.292,0,0.010) -- cycle;
    \fill[black!85] (-0.176,0,0.008) -- (-0.176,0,0.030) -- (-0.164,0,0.030) -- (-0.164,0,0.008) -- cycle;
    \fill[black!40, opacity=0.35] plot[domain=0:360, samples=48, variable=\t] ({-0.17+0.034*cos(\t)},{0.034*sin(\t)},0.033) -- cycle;
    \draw[black!85, line width=0.5pt, line cap=round] (-0.200,0,0.033) -- (-0.140,0,0.033);
    \draw[black!85, line width=1.0pt, line cap=round] (0.03,0,0.010) -- (-0.29,0,0.010);
    \foreach \s/\c in {0.020/70, -0.020/85}{
      \fill[black!\c]
        (0.255,\s,-0.002)
        .. controls (0.255,\s,0.015) and (0.245,\s,0.024) .. (0.222,\s,0.025)
        -- (0.11,\s,0.027)
        .. controls (0.07,\s,0.026) and (0.045,\s,0.019) .. (0.015,\s,0.012)
        -- (0.015,\s,0.006)
        .. controls (0.05,\s,-0.006) and (0.09,\s,-0.021) .. (0.13,\s,-0.023)
        -- (0.215,\s,-0.022)
        .. controls (0.245,\s,-0.021) and (0.255,\s,-0.013) .. cycle;
    }
    \fill[black!45]
      (0.247,-0.020,0.010) .. controls (0.241,-0.020,0.021) and (0.226,-0.020,0.026) .. (0.196,-0.020,0.027)
      -- (0.193,-0.020,0.013) .. controls (0.212,-0.020,0.011) and (0.235,-0.020,0.010) .. cycle;
    \coordinate (vtolcam) at (0.212,-0.006,-0.026);
    \coordinate (-cam) at (0.212,-0.006,-0.037);
    \fill[black!60, draw=black!85, line width=0.3pt, line join=round]
      (0.185,0,0.034) -- (0.165,0.30,0.044) .. controls (0.160,0.31,0.044) and (0.120,0.31,0.044) .. (0.112,0.29,0.044)
      -- (0.080,0,0.034) -- (0.112,-0.29,0.044) .. controls (0.120,-0.31,0.044) and (0.160,-0.31,0.044) .. (0.165,-0.30,0.044) -- cycle;
    \foreach \s in {0.11,-0.11}{
      \begin{scope}[shift={(0.180,\s,0.038)}, x={(\Ax cm,\Ay cm)}, z={(\Bx cm,\By cm)}]
        \fill[black!80] (-0.010,0,-0.005) -- (0.026,0,-0.006) .. controls (0.031,0,-0.006) and (0.031,0,0.006) .. (0.026,0,0.006) -- (-0.010,0,0.005) -- cycle;
        \fill[black!85] (0.026,0,-0.0065) -- (0.035,0,-0.0065) -- (0.035,0,0.0065) -- (0.026,0,0.0065) -- cycle;
        \fill[black!40, opacity=0.35] plot[domain=0:360, samples=48, variable=\t] (0.038,{0.040*cos(\t)},{0.040*sin(\t)}) -- cycle;
        \draw[black!85, line width=0.5pt, line cap=round] (0.038,0,-0.034) -- (0.038,0,0.034);
      \end{scope}
    }
    \fill[black!85] (-0.233,0,0.010) -- (-0.278,-0.070,0.060) -- (-0.298,-0.070,0.060) -- (-0.292,0,0.010) -- cycle;
  \end{scope}
  \fill[black] (vtolcam) circle (0.011);
  \fill[black!40] ([shift={(0.004,-0.003)}]vtolcam) circle (0.0032);
  },
}
\begin{tikzpicture}
  \begin{axis}[bsp base, name=scene,
      width=0.80\columnwidth, height=1.75cm,
      xmin=0, xmax=8, ymin=0, ymax=2.05,
      axis lines=none, clip=false]
    \bspshade{2}{5}{0}{2.05}
    \addplot[domain=0:8.5, samples=170, bspPlan!45, line width=0.6pt]
      {0.17+0.035*sin(deg(5*x))};
    \filldraw[fill=black!10, draw=black!70, line width=0.5pt]
      (axis cs:6.75,0.52) -- (axis cs:6.98,0.12) -- (axis cs:8.5,0.12)
      -- (axis cs:8.5,0.52) -- cycle;
    \filldraw[fill=black!22, draw=black!70, line width=0.5pt]
      (axis cs:8.05,0.52) rectangle (axis cs:8.5,0.98);
    \draw[line width=1.5pt, black!85] (axis cs:6.80,0.52) -- (axis cs:8.05,0.52);
    \coordinate (spot) at (axis cs:7.6,0.53);
    \addplot[domain=0:2,   samples=30, black!65, line width=0.7pt]
      {0.57+1.3*(1-(x/7.6)^2)};
    \addplot[domain=2:5,   samples=40, bspRej!85!black, line width=0.9pt, densely dashed]
      {0.57+1.3*(1-(x/7.6)^2)};
    \addplot[domain=5:7.45, samples=30, black!65, line width=0.7pt]
      {0.57+1.3*(1-(x/7.6)^2)};
    \coordinate (p1) at (axis cs:1.0,1.847);
    \coordinate (p2) at (axis cs:3.6,1.578);
    \coordinate (p3) at (axis cs:6.0,1.080);
    \node[bsp label, anchor=south] at (axis cs:1.0,2.05) {deck in view};
    \node[bsp label, anchor=south, text=bspRej!70!black] at (axis cs:3.5,2.05) {dropout during flare};
    \node[bsp label, anchor=south] at (axis cs:6.5,2.05) {deck reacquired};
  \end{axis}
  \pic[rotate=-3] (A1) at (p1) {vtol};
  \pic[rotate=20] (A2) at (p2) {vtol};
  \pic[rotate=5] (A3) at (p3) {vtol};
  \draw[bspAcc, densely dotted, line width=0.7pt] (A1-cam) -- (spot);
  \draw[bspAcc, densely dotted, line width=0.7pt] (A3-cam) -- (spot);
  \fill[bspRej, opacity=0.35] (A2-cam) -- ++(-4:1.55) arc[start angle=-4, end angle=12, radius=1.55] -- cycle;
  \draw[bspRej!85!black, line width=0.4pt] (A2-cam) -- ++(-4:1.55);
  \draw[bspRej!85!black, line width=0.4pt] (A2-cam) -- ++(12:1.55);
  \node[bsp panel, anchor=north east] at ([xshift=-0.5cm]scene.north west) {(a)};

  \begin{axis}[bsp base, name=err,
      at={(scene.south west)}, anchor=north west, yshift=-0.12cm,
      width=0.80\columnwidth, height=2.75cm,
      xmin=0, xmax=8, ymin=-3.6, ymax=3.6,
      xtick={0,1,...,8}, ytick=\empty,
      xlabel={time [s]}, ylabel={along-deck error},
      legend columns=2,
      legend style={at={(0.5,-0.30)}, anchor=north, fill=none,
                    /tikz/every even column/.append style={column sep=8pt}},
      transpose legend]
    \bspshade{2}{5}{-3.6}{3.6}
    \draw[black!35, line width=0.4pt] (axis cs:0,0) -- (axis cs:8,0);
    \addplot[name path=Vhi, draw=none, forget plot] table[x=t, y=Vhi] {\bspdata/teaser.dat};
    \addplot[name path=Vlo, draw=none, forget plot] table[x=t, y=Vlo] {\bspdata/teaser.dat};
    \addplot[bsp estband] fill between[of=Vhi and Vlo];
    \addplot[bsp estedge] table[x=t, y=Vhi] {\bspdata/teaser.dat};
    \addplot[bsp estedge] table[x=t, y=Vlo] {\bspdata/teaser.dat};
    \addplot[name path=Phi, draw=none, forget plot, restrict x to domain=2:8]
      table[x=t, y=hi] {\bspdata/teaser.dat};
    \addplot[name path=Plo, draw=none, forget plot, restrict x to domain=2:8]
      table[x=t, y=lo] {\bspdata/teaser.dat};
    \addplot[bsp planband] fill between[of=Phi and Plo];
    \addplot[bsp plan, forget plot, restrict x to domain=2:8] table[x=t, y=m] {\bspdata/teaser.dat};
    \addplot[bsp truth, forget plot] table[x=t, y=e] {\bspdata/teaser.dat};
    \addplot[only marks, mark=|, mark size=2.2pt, bspAcc, line width=0.6pt, forget plot]
      table[x=t, y expr=-3.32] {\bspdata/teaser_acc.dat};
    \draw[bspPlan, line width=0.5pt] (axis cs:2,3.6) -- (axis cs:2,2.9)
      node[anchor=north, font=\scriptsize, text=bspPlan, inner sep=1pt] {$t_c$};
    \addlegendimage{bsp truth}
    \addlegendentry{realized error $e$}
    \addlegendimage{legend image code/.code={
        \fill[bspEst, opacity=0.22] (0cm,-0.09cm) rectangle (0.5cm,0.09cm);
        \draw[bspEst, dashed, line width=0.5pt] (0cm,0.09cm) -- (0.5cm,0.09cm);}}
    \addlegendentry{estimator report $0\pm2\sqrt{V}$}
    \addlegendimage{legend image code/.code={
        \fill[bspPlan, opacity=0.22] (0cm,-0.09cm) rectangle (0.5cm,0.09cm);
        \draw[bspPlan, line width=0.9pt] (0cm,0cm) -- (0.5cm,0cm);}}
    \addlegendentry{planner prediction $m\pm2\sqrt{P}$}
    \addlegendimage{only marks, mark=|, mark size=2.2pt, bspAcc, line width=0.6pt}
    \addlegendentry{accepted correction}
  \end{axis}
  \node[bsp panel, anchor=north east] at ([xshift=-0.5cm]err.north west) {(b)};
\end{tikzpicture}
    \caption{Predictable estimator error during a deck landing (illustrative one-dimensional deck-relative model, with command-blind estimation) (a) The vehicle conducts a flare maneuver, pitching the deck out of the camera's field of view between 2 and 5 s. (b) The estimator does not model the commanded flare, so its prediction accumulates error during the dropout while its modeled error distribution remains centered at zero with widening covariance \(V\) (gray). The planner issued the command and knows the estimator is command-ignorant, so predicts the evolution of estimator-error conditioned on its own information $\mathcal F$. Planner-predicted estimator error $m$ tracks the realized estimator error (black) at a much narrower covariance $P$ (blue). After visibility is restored, corrections (green) are accepted stochastically. In this example, accepted and uncorrected outcomes have different predicted error means, whose spread temporarily widens \(P\). Colors are consistent across figures, where gray represents estimator reports, blue represents planner-conditioned moments, green indicates correction acceptance, and shading for intervals with no available corrections.} 
    \label{fig:teaser}
\end{figure}

\subsection{Related work}

Closed loop uncertainty prediction and belief-space planning approaches are well established, though typically assume that a specified estimator model is consistent with the planning model \cite{platt2010, zotero-item-71}. This model uncertainty is centered on the modeled posterior mean and zero-mean estimator error, despite the potential existence of conditionally nonzero bias in the deployed estimator. Approaches like LQG-MP provide a framework for predicting joint dispersion between truth and the published estimates of a Kalman filter \cite{lqg}, though they do not model the planner-conditioned error of externally deployed estimators when the information sets of the planner and estimator differ.

Observation intermittency provides another source of uncertainty that complicates planning problems \cite{sinopoli2004, yan2022, piazzoni2024}. Relevant treatments include scalar information-approximation approaches, like smoothing sensing-discontinuities \cite{patil2014a} and measurement-probability information weighting \cite{rahman2021}, and approaches which retain acquisition randomness \cite{chaves2015, bopardikar}. While it is established that sensing geometry and random acquisition affect future uncertainty, accepted and rejected corrections yield different predicted error means, which contribute to additional covariance, changing the prospective task risk.

Various planning approaches consider modular estimators and estimator quality. Some employ learned future-error costs \cite{kayalibay} while others propagate perception uncertainty explicitly \cite{nagami2024,pampc}. Instead of propagating an error cost, we propagate conditional estimator-error, modeling the first and second moments of estimator error conditioned on the planner's information.

Our work builds upon prior analyses of estimator-controller information asymmetry \cite{trodden2023,schenato2007} and established machinery for mixture moment matching \cite{bar-shalom1975}, with which we develop a model of planner-conditioned estimator error, its effect on stochastic correction outcomes, a closed moment propagation of that error through correction, and a task risk-metric which enables a planner to reason about trajectory impacts on estimator performance.
\subsection{Contributions}

We propose a belief-space planning approach which derives a model of the first and second moments of estimator error. This model is used to propagate the estimator error along candidate trajectories. The planner predicts the probability that usable corrections will be available to the estimator and moment-matches the first two error moments over both correction and no-correction outcomes at each step.  This predicted error is then projected into task coordinates, from which the planner computes a task-weighted estimation risk metric. These predictions are then used by the receding-horizon planner to select trajectories. 

In this work, we make the following claims:

\begin{itemize}

    \item
    We show that, under different information sets, even estimators that are
    minimum mean-squared error  (MMSE)-optimal under their own information may have a nonzero
    planner-conditioned error mean, as established in
    Proposition~\ref{prop:planner_conditioned_error} and evaluated empirically
    through the error prediction tests of
    Section~\ref{sec:predictive_results};

    \item
    We develop a moment-matched recursion for estimator-error propagation
    under intermittent corrections, capturing the additional covariance
    induced by separation between the accepted and rejected error means that
    probability-based noise inflation omits, while recovering the familiar
    modified-Riccati recursion in the zero-mean limit;

    \item
    We formulate a task-weighted planner that penalizes predicted estimator-error
    moments along candidate trajectories, identify error models for three estimator architectures, and evaluate closed-loop ship-deck landing of a tilt-rotor VTOL in simulation.

\end{itemize}

We find that estimator error contains substantial planner-predictable structure when the estimator is ignorant of task-relevant information the planner holds. Prediction gains vary with forecast horizon and operating regime, though in 16 paired closed-loop trials, the proposed planner consistently improved median task performance compared to traditional covariance-only belief-space planning.


\section{Information Structure and Predictable Estimator Error}
\label{sec:information}

At some future time \(k\), let \(x_k\) denote the physical state
and \(u_k\) the applied command. The estimator reports an estimated state
\(\hat x_k\) and covariance \(V_k\), and its estimation error is
\[
    e_k = x_k\ominus\hat x_k,
\]
where \(\ominus\) is the difference operator on the relevant state manifold.
Let \(\mathcal H_k^E\) denote the \(\sigma\)-algebra generated by all
information available internally to the estimator at time \(k\). Each
published estimator output \(I_k\), including \(\hat x_k\), \(V_k\),
and any other quantities published by the estimator and used by the planner-side error model, is
\(\mathcal H_k^E\)-measurable.

The information available to the planner is
\begin{equation}
    \mathcal F_k
    =
    \sigma\!\left(I_{0:k},u_{0:k-1}\right).
    \label{eq:planner_information}
\end{equation}
The planner-conditioned estimator-error moments are
\begin{equation}
\begin{aligned}
    m_k &= \mathbb E[e_k\mid\mathcal F_k],\\
    P_k &= \operatorname{Cov}(e_k\mid\mathcal F_k).
\end{aligned}
\label{eq:error_moments}
\end{equation}
In general, \(P_k\) need not equal the covariance \(V_k\) reported by the
estimator.

\begin{proposition}[Planner-conditioned estimator error]
\label{prop:planner_conditioned_error}
Let
\[
    X=Y+BU,\qquad
    \hat X=\mathbb E[X\mid\mathcal H^E],\qquad
    e=X-\hat X,
\]
where \(X,Y\in L^2(\Omega;\mathbb R^n)\),
\(U\in L^2(\Omega;\mathbb R^p)\), and \(B\) is deterministic. Then
$
    \mathbb E[e\mid\mathcal G]=0
    \text{ for every }
    \mathcal G\subseteq\mathcal H^E.
$

If
$\sigma(U)
    \subseteq
    \mathcal F
    \subseteq
    \mathcal H^E\vee\sigma(U),
$
\(\mathbb E[U\mid\mathcal H^E]\) is \(\mathcal F\)-measurable, and
\[
    \mathbb E[Y\mid\mathcal H^E\vee\sigma(U)]
    =
    \mathbb E[Y\mid\mathcal H^E],
\]
then
\begin{equation}
    \mathbb E[e\mid\mathcal F]
    =
    B\left(
        U-\mathbb E[U\mid\mathcal H^E]
    \right).
    \label{eq:planner_conditioned_error}
\end{equation}
\end{proposition}

\begin{proof}
The MMSE property gives
\(\mathbb E[e\mid\mathcal H^E]=0\), and the first claim follows from the
tower property. Moreover,
\[
    \hat X
    =
    \mathbb E[Y\mid\mathcal H^E]
    +
    B\mathbb E[U\mid\mathcal H^E].
\]
Conditioning \(X-\hat X\) on \(\mathcal F\) under the stated assumptions
then gives \eqref{eq:planner_conditioned_error}.
\end{proof}

Thus, an estimator that is MMSE-optimal given \(\mathcal H^E\) can still
have nonzero error mean conditioned on \(\mathcal F\) when the planner has
information relevant to state evolution that is not fully determined by
the estimator's information. Whether this occurs for a particular
estimator and operating regime is tested from data in
Section~\ref{sec:identification}.

Estimates are formed using observations of the environment that arrive stochastically. In part, this stochasticity arises from sensing geometry, with the probability of successful measurement decreasing at the observer's field of view boundaries. However, the observation of a measurement does not always lead to estimate correction. In cases where an estimator gates its accepted measurements on measurement innovation or the cadence in which measurements arrive, the probability that a measurement will be observed and that its observation will correct the estimator's published information may partially depend on factors beyond a planner's direct control \cite{or2026}. We call the sensing and application of a measurement to a state estimate a measurement acceptance. Measurement acceptance comprises one branch of a two-branch "tree" of possible sensing outcomes that the planner must consider at each potential future correction opportunity.


\section{A Planner-conditioned Model of Estimator Error}
\label{sec:error_model}
Throughout this formulation, let $c$ be the current planning time--the instant at which the planner forms a candidate control sequence and predicts forward--and $k$ denote a generic future discrete-time index along that candidate rollout. We use the subscript $k|c$ to denote the future time $k$, predicted using the information available at planning time $c$.

\subsection{Planner information}
\label{sec:interface}
The planner does not require access to the details of an estimator's internal belief representation or propagation mechanisms in order to predict its error. However, it does require the estimator to provide a mechanism by which to initialize the error moment predictors,  a model for between-correction error flow, the effect of accepted corrections on the estimator's predictions, the probability that a correction is accepted by the estimator, and the range over which the estimator's behavior is considered valid.

\subsection{Initialization and compensation}
\label{sec:compensation}
At the beginning of planning time $c$, the estimator-error moments are initialized using the set of all previously published estimator information $(m_c,P_c) = \mathcal I_E(I_{0:c})$ using some declared initialization procedure $\mathcal I_E$. The planner uses this initialized error prediction to compensate the estimator's published state, yielding a compensated planner-side state estimate $\check x_c$, such that

\begin{equation}
    \check x_c
    =
    \hat x_c \oplus m_c.
    \label{eq:compensation}
\end{equation}
where $\oplus$ is the retraction operator on the relevant manifold.

\subsection{Estimator-error propagation between corrections}
\label{sec:error_flow}

We model estimator-error propagation between observation-based corrections
over integration interval \(\delta\) as
\begin{equation}
    e_{k+1}^-
    =
    F e_k^+
    +
    \delta\left(
        b+B\phi_k+\Lambda_z z_k
    \right)
    +
    \eta_k,
    \qquad
    F=I+\delta A,
    \label{eq:error_flow}
\end{equation}
where \(A\) and \(b\) characterize autonomous error evolution,
\(\phi_k\) collects candidate-dependent features available to the planner,
and \(z_k\) represents the retained temporally correlated residual modes.
The remaining residual satisfies
\[
    \mathbb E[\eta_k\mid\mathcal F_k]=0.
\]

Let \(r_k\in[0,1]^{n_r}\) denote correction recency, with
$
    \bar r_k
    =
    \mathbb E[r_k\mid\mathcal F_k],$ and $
    \Sigma_k^{rr}
    =
    \operatorname{Cov}(r_k\mid\mathcal F_k)$.
Its contribution to the remaining flow uncertainty is
\begin{equation}
\begin{aligned}
    \operatorname{Cov}(\eta_k\mid\mathcal F_k)
    =
    \delta\Bigg[
        Q_0
        +
        \sum_{j=1}^{n_r}
        \Big(
            (1-\bar r_{j,k})Q_j^{(1)}
            \\
            +
            \big[
                (1-\bar r_{j,k})^2
                +
                \Sigma^{rr}_{jj,k}
            \big]
            Q_j^{(2)}
        \Big)
    \Bigg],
\end{aligned}
\label{eq:flow_covariance}
\end{equation}
where $Q_0$ is the baseline residual uncertainty, $Q_j^{(1)}$ is the uncertainty which increases linearly with the staleness of correction channel $j$, and $Q_j^{(2)}$ is the residual uncertainty that increases quadratically with staleness.
The quadratic term uses the full second-moment of stochastic recency, since $\E[(1-r)^2|\mathcal F_k] = (1-\bar r_j)^2 + \Sigma_{jj,k}^{rr}$

The colored residual is represented by the stable autoregressive model
$
    z_{k+1}
    =
    R_z z_k+\epsilon_k^z$ for 
    $\rho(R_z)<1,$ with centered innovation \(\epsilon_k^z\), where $\rho(R_z)<1$ is a constraint to ensure the stable decay of the latent residual state. Between corrections, recency
evolves according to
\begin{equation}
    r_{k+1}^-
    =
    D_r r_k^+.
    \label{eq:recency_flow}
\end{equation}

Define the augmented recency-error state
$
    s_k
    =
    [
        e_k^\top, 
        r_k^\top, 
        z_k^\top
    ]^{\top}.
$
For a fixed candidate command sequence, the between-correction estimator-error propagation model can
then be written in affine form as
\begin{equation}
    s_{k+1}^-
    =
    \mathcal A_k s_k^+
    +
    c_k
    +
    \omega_k,
    \label{eq:augmented_error_flow}
\end{equation}
where $\mathbb E[\omega_k\mid\mathcal F_k]=0,$ and $\operatorname{Cov}(\omega_k\mid\mathcal F_k)=\Omega_k.$
Its first two moments propagate as
\begin{equation}
\begin{aligned}
    \mu_{k+1}^-
    &=
    \mathcal A_k\mu_k^+
    +
    c_k,
    \\
    \Sigma_{k+1}^-
    &=
    \mathcal A_k\Sigma_k^+\mathcal A_k^\top
    +
    \Omega_k,
\end{aligned}
\label{eq:augmented_moment_flow}
\end{equation}
and the estimator-error moments \(m_k\) and \(P_k\) are the \(e\)-blocks of
\(\mu_k\) and \(\Sigma_k\), respectively.
\subsection{Stochastic corrections}
\label{sec:correction_model}

Let \(a_{k|c}\in\{0,1\}\) denote whether a usable correction occurs at
rollout step \(k\). We model $a_{k|c}\sim\operatorname{Bernoulli}(q_{k|c})$
with
\begin{equation}
    q_{k|c}
    =
    \mathcal Q(
        \bar x_{k|c},
        r_{k|c},
        I_c
    ).
    \label{eq:correction_probability}
\end{equation}
Here \(\mathcal Q\) is a planner-side correction-probability model accounting
for sensing cadence, trajectory-dependent sensing geometry, and conditional
estimator acceptance. 

Define the error of the accepted-correction branch as
\begin{equation}
    e^{\rm acc} = Je^- + \epsilon^c, \qquad \E[\epsilon^c|\mathcal F] = 0, \quad \operatorname{Cov}(\epsilon^c|\mathcal F) = W.
    \label{eq:accepted_correction}
\end{equation}
for error-transition map $J$ and remaining correction error $\epsilon^c$. For linear updates, $J = I-KH$, where $K$ is the correction gain and $H$ is the measurement Jacobian. When observations are accepted, the recency component associated with an accepted measurement type is reset to $r^+_j=1$, while it continues decaying in the no-observation branch.


\section{Moment Propagation Under Intermittent Corrections}
\label{sec:intermittent_corrections}

Under intermittent observations of the task, future observation-based estimation corrections have state-dependent stochasticity. At each $t_k$, either an observation can be accepted with probability $q$, or the estimate can remain uncorrected and error can continue to propagate under the between-correction flow model. While branching each candidate over the possible observation outcomes would enable consideration of this stochasticity, this approach is computationally intractable for real-time control problems. We instead show that the first two moments of planner-conditioned estimator-error under stochastic observations can be matched exactly over the correction outcomes. When the prior of the predicted mean error is nonzero, this moment-matching introduces an additional covariance term, distinguishing our approach from other zero-mean treatments that rely instead on scalar noise-inflation.

\subsection{Rollout assumptions}
\label{sec:rollout_assumptions}
We declare the following assumptions under which we formulate our moment-matching approach.
\begin{assumption}[Rollout closure]
\label{ass:rollout_closure}
For each candidate rollout:
\begin{enumerate}
    \item The local flow and correction models are evaluated only within
    their validated range \(\mathcal V_E\).

    \item Correction occurrence is conditionally independent of the
    realized estimator error given planner information,
    \[
        \Prob(
            a_k=1
            \mid
            e_k^-,
            \mathcal F_k
        )
        =
        \Prob(
            a_k=1
            \mid
            \mathcal F_k
        )
        =
        q_k.
    \]
    \item Candidate command sequences are evaluated open loop between
    replanning updates.
\end{enumerate}
\end{assumption}
\subsection{Moment matching over correction outcomes}
\label{sec:moment_matching}

Under the stochastic-acceptance model of Section~\ref{sec:correction_model}, let event $a_k = 1$ denote correction acceptance at time $k$. Likewise, $a_k = 0$ denotes rejection. The first and second moments of estimator-error are defined for each branch as
\[
\begin{aligned}
    m_{\rm rej}
    &=
    \E[e^+\mid a=0,\mathcal F],
    &
    P_{\rm rej}
    &=
    \Cov(e^+\mid a=0,\mathcal F),
    \\
    m_{\rm acc}
    &=
    \E[e^+\mid a=1,\mathcal F],
    &
    P_{\rm acc}
    &=
    \Cov(e^+\mid a=1,\mathcal F).
\end{aligned}
\]
An exact first-two-moment closure of the branches arises from the laws of total expectation and covariance, yielding the posterior error correction model:

\begin{equation}
\begin{aligned}
    m^+
    &=
    (1-q)m_{\rm rej}
    +
    q m_{\rm acc},
    \\
    P^+
    &=
    (1-q)P_{\rm rej}
    +
    qP_{\rm acc}
    \\
    &\quad+
    q(1-q)
    (m_{\rm acc}-m_{\rm rej})
    (m_{\rm acc}-m_{\rm rej})^\top .
\end{aligned}
\label{eq:mixture_moments}
\end{equation}
The first two covariance terms model within-branch uncertainty, and the final term models the spread between the branch means. This final covariance term is zero precisely when the accepted and rejected branch means coincide and $0<q<1$.
\subsection{Effect of predictable estimator error}
\label{sec:biased_correction}

We apply this mixture to the error model established in Section~\ref{sec:correction_model}. In particular, define the $a_k=0$ branch moments $m_{\rm rej}=m^-$ and $P_{\rm rej}=P^-$. The $a_k=1$ branch is defined by \eqref{eq:accepted_correction}, with resulting branch moments $m_{\rm acc}=Jm^-$ and $P_{\rm acc}=JP^-J^\top+W$.

\begin{proposition}[Intermittent correction of biased error]
\label{prop:correction_gap}

Let an accepted correction occur with probability \(q\in(0,1)\) under
Assumption~\ref{ass:rollout_closure}, and define
$
    v=(J-I)m^-.
$
Then the post-event estimator-error mean is
\begin{equation}
    m^+
    =
    \big[(1-q)I+qJ\big]m^-,
    \label{eq:intermittent_mean}
\end{equation}
and its covariance is
\begin{equation}
    P^+(m^-)
    =
    P^+(0)
    +
    q(1-q)vv^\top,
    \label{eq:mean_induced_spread}
\end{equation}
where
$
    P^+(0)
    =
    (1-q)P^-
    +
    q\left(JP^-J^\top+W\right).$

\end{proposition}

\begin{proof}
Substituting $m_{\rm rej} = m^-$, $P_{\rm rej} = P^-$, $m_{\rm acc} = Jm^-$, and $P_{\rm acc} = JP^-J^\top+W$ into \eqref{eq:mixture_moments} gives \eqref{eq:intermittent_mean}--\eqref{eq:mean_induced_spread}. 
\end{proof}
Therefore, in general, scalar inflation alone cannot reconstruct the stochastic-acceptance uncertainty of \eqref{eq:mixture_moments}.
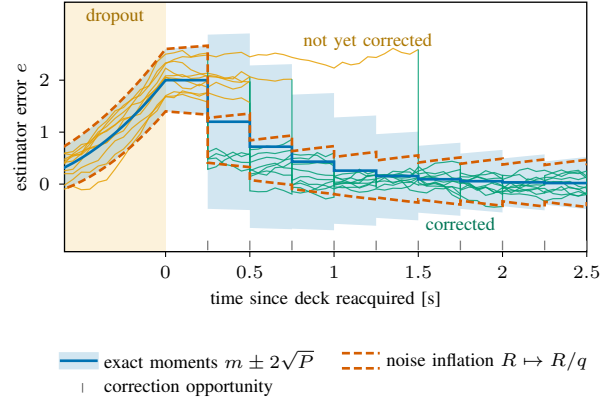
\begin{figure}[t]
    \centering
\begin{tikzpicture}
  \begin{axis}[bsp base,
      width=0.80\columnwidth, height=3.3cm,
      xmin=-0.6, xmax=2.5, ymin=-1.3, ymax=3.5,
      xtick={0,0.5,...,2.5}, ytick={0,1,2},
      xlabel={time since deck reacquired [s]}, ylabel={estimator error $e$},
      unbounded coords=jump,
      legend columns=2,
      legend style={at={(0.5,-0.36)}, anchor=north, fill=none,
                    /tikz/every even column/.append style={column sep=8pt}}]
    \bspshade{-0.6}{0}{-1.3}{3.5}
    \addplot[name path=hi, draw=none, forget plot] table[x=t, y=hi] {\bspdata/correction_futures.dat};
    \addplot[name path=lo, draw=none, forget plot] table[x=t, y=lo] {\bspdata/correction_futures.dat};
    \addplot[draw=none, fill=bspPlan, fill opacity=0.18, forget plot] fill between[of=hi and lo];
    \pgfplotsinvokeforeach{1,...,10}{
      \addplot[bspRej, line width=0.45pt, opacity=0.85, forget plot]
        table[x=t, y=u#1] {\bspdata/correction_futures.dat};
      \addplot[bspAcc, line width=0.45pt, opacity=0.85, forget plot]
        table[x=t, y=c#1] {\bspdata/correction_futures.dat};
    }
    \addplot[bspPlan, line width=1.0pt, forget plot] table[x=t, y=m] {\bspdata/correction_futures.dat};
    \pgfplotsinvokeforeach{hit,lot}{
      \addplot[bspInfl, line width=1.0pt, dash pattern=on 3pt off 1.5pt, forget plot]
        table[x=t, y=#1] {\bspdata/correction_futures.dat};
    }
    \addplot[only marks, mark=|, mark size=1.8pt, black!50, line width=0.5pt, forget plot]
      table[x=t, y expr=-1.18] {\bspdata/correction_opps.dat};
    \node[bsp label, anchor=north, text=bspRej!70!black] at (axis cs:-0.3,3.45) {dropout};
    \node[bsp label, anchor=west, text=bspRej!75!black] at (axis cs:0.8,2.75) {not yet corrected};
    \node[bsp label, anchor=center, text=bspAcc!75!black] at (axis cs:1.75,-0.76) {corrected};
    \addlegendimage{legend image code/.code={
        \fill[bspPlan, opacity=0.18] (0cm,-0.1cm) rectangle (0.5cm,0.1cm);
        \draw[bspPlan, line width=1.0pt] (0cm,0cm) -- (0.5cm,0cm);}}
    \addlegendentry{exact moments $m\pm2\sqrt{P}$}
    \addlegendimage{legend image code/.code={
        \draw[bspInfl, line width=1.0pt, dash pattern=on 3pt off 1.5pt] (0cm,0.1cm) -- (0.5cm,0.1cm);
        \draw[bspInfl, line width=1.0pt, dash pattern=on 3pt off 1.5pt] (0cm,-0.1cm) -- (0.5cm,-0.1cm);}}
    \addlegendentry{noise inflation $R\mapsto R/q$}
    \addlegendimage{only marks, mark=|, mark size=1.8pt, black!50, line width=0.5pt}
    \addlegendentry{correction opportunity}
  \end{axis}
\end{tikzpicture}
    \caption{Intermittent correction of biased error (Proposition \ref{prop:correction_gap}) using a scalar example with $q=0.5$ and $J=0.2$. Sensor loss leaves the estimator-error with a predictable mean $m^-\approx 2$. Once observation is restored, each correction opportunity (ticks) is accepted with probability $q$, so the possible futures branch into those already corrected (green) and those not yet corrected (orange). The exact moments (blue) cover both groups. Additional covariance spread $P$ is induced by $q(1-q)vv^\top$, which vanishes for $m^- = 0$. Noise inflation $R\mapsto R/q$ (dashed band) treats each opportunity as a weaker but certain correction, so it falls between the groups. After the first correction opportunity, its $2\sigma$ band contains under 2\% of possible futures, against over 99\% for exact moments. Ten representative paths of $2\times10^4$ simulations are drawn.}
    \label{fig:correction}
\end{figure}

\subsection{Zero-mean limit}
\label{sec:zero_mean_special_case}
Setting \(m^-=0\) in Proposition~\ref{prop:correction_gap} gives
\begin{equation}
\begin{aligned}
     m^+&=0,
    \\
    P^+
    &=
    (1-q)P^-
    +
    q\left(
        JP^-J^\top
        +W
    \right).
\end{aligned}
    \label{eq:zero_mean_special_case}
\end{equation}For a linear-Gaussian correction with
$
    J=I-KH$, $
    W=KRK^\top$, $
    S^c=0$, $P^- = V^-$,
and
$
    K
    =
    V^-H^\top
    \left(
        HV^-H^\top+R
    \right)^{-1},
$
\eqref{eq:zero_mean_special_case} reduces to
\begin{equation}
    P^+
    =
    P^-
    -
    qP^-H^\top
    \left(
        HP^-H^\top+R
    \right)^{-1}
    HP^- .
    \label{eq:modified_riccati_update}
\end{equation}

This is precisely a modified-Riccati update, using the familiar intermittent-measurement correction \cite{sinopoli2004}. Our proposed propagation therefore contains the conventional zero-mean recursion as the special case for $m^- = 0$, with additional terms appearing only when the rejected and accepted branches have different predicted means.

\subsection{Comparison with probability-based noise inflation}
\label{sec:noise_inflation_comparison}

For the stochastic correction, define the exact expected correction map
$\bar J \triangleq (1-q)I+qJ = I-qKH$, hence
$m^+ = \bar Jm^-$. Under noise-inflated update $R\mapsto R/q$, define
$\widetilde J\triangleq I-\widetilde KH$. Probability-based noise inflation \(R\mapsto R/q\) gives
\begin{equation}
    \widetilde K
    =
    qP^-H^\top
    \left(
        qHP^-H^\top+R
    \right)^{-1}.
    \label{eq:noise_inflated_gain}
\end{equation}
For \(0<q<1\), \(P^-\succ0\), \(R\succ0\), and full-row-rank \(H\),
\begin{equation}
    \widetilde K\neq qK,
    \qquad
    \widetilde J\neq\bar J.
    \label{eq:gain_mismatch}
\end{equation}
For \eqref{eq:gain_mismatch}, equality
\(\widetilde K=qK\) would imply $(1-q)HP^-H^\top=0$,
which contradicts \(HP^-H^\top\succ0\).
Its covariance update is similarly
\begin{equation}
    \widetilde P^+
    =
    \widetilde J P^-\widetilde J^\top
    +
    \widetilde K
    \left(\frac{R}{q}\right)
    \widetilde K^\top,
    \label{eq:noise_inflated_covariance}
\end{equation}
which is independent of $m^-$. In contrast,
\begin{equation}
    P^+(m^-)-P^+(0)
    =
    q(1-q)
    (J-I)m^-
    (m^-)^\top
    (J-I)^\top .
    \label{eq:bias_dependent_covariance}
\end{equation}

Therefore, since noise inflation generally mismatches the expected mean update and its covariance update is bias-independent, it cannot represent bias-dependent recursion for an arbitrary $m^-$.
\subsection{Recursive rollout closure}
\label{sec:recursive_closure}
The proposed stochastic-correction recursion remains closed for the augmented error-recency state $\chi_k = [e_k, r_k]^\top$. The corresponding branch-conditioned moments are $
    (\mu_{\chi,\rm rej},\Sigma_{\chi,\rm rej})
    $ and $
    (\mu_{\chi,\rm acc},\Sigma_{\chi,\rm acc}),
$ which propagate according to
\begin{equation}
\begin{aligned}
    \mu_\chi^+
    &=
    (1-q)\mu_{\chi,\rm rej}
    +
    q\mu_{\chi,\rm acc},
    \\
    \Sigma_\chi^+
    &=
    (1-q)\Sigma_{\chi,\rm rej}
    +
    q\Sigma_{\chi,\rm acc}
    \\
    &\quad+
    q(1-q)
    (\mu_{\chi,\rm acc}-\mu_{\chi,\rm rej})
    (\mu_{\chi,\rm acc}-\mu_{\chi,\rm rej})^\top .
\end{aligned}
\label{eq:joint_mixture_moments}
\end{equation}

\section{Task-Weighted Estimation Risk and Receding-Horizon Planning}
\label{sec:planning}

\subsection{Task projection}
\label{sec:task_projection}
Define the projection of estimator-error into task-relevant coordinates as
\begin{equation}
    \delta y_k
    \simeq
    G_k e_k,
    \label{eq:task_projection}
\end{equation}
where $G$ is the Jacobian which maps estimator-coordinates into task-space. Let $\tau_i,i\in [1,n_y]$ be the set of tolerances on each task axis, which defines the successful task-completion set. The task-normalized cost metric can be constructed as:

\begin{equation}
    M
    =
    \operatorname{diag}
    \left(
        \tau_1^{-2},
        \ldots,
        \tau_{n_y}^{-2}
    \right).
    \label{eq:tolerance_metric}
\end{equation}

\subsection{Task-weighted estimation risk}
\label{sec:estimation_risk}
Compensation of the current physical nominal does not set $m_k=0$, and so $m_k$ remains the predicted mean of the external estimator error, determining future correction outcomes and their task-space effect. 

Projecting the planner-conditioned error-moments into task-space enables the computation of task-normalized estimation risk. Let the expected task-normalized estimation risk metric be defined as $\mathcal R_E \triangleq \E[\lVert G_ke_k\rVert^2_M | \mathcal F_k]$. Hence
\begin{equation}
    \mathcal R_E(m_k,P_k)
    =
    (G_km_k)^\top M(G_km_k)
    +
    \operatorname{tr}
    \left(
        MG_kP_kG_k^\top
    \right).
    \label{eq:estimation_risk}
\end{equation}
This risk metric can be decomposed, such that $\mathcal R_E = \mathcal R_m + \mathcal R_P$ for the predictable-bias risk $\mathcal R_m = \lVert G_km_k\rVert^2_M$ and covariance risk $\mathcal R_P = \tr(MG_kP_kG_k^\top)$. This identity is exact for the supplied moments $(m,P)$ and fixed, $\mathcal F_k$-measurable projection $G_k$, requiring only finite second moments.

\subsection{Receding-horizon planning objective}
\label{sec:planning_objective}
Let $\mathbf u_c
    =
    \left(
        u_{c|c},
        \ldots,
        u_{c+N-1|c}
    \right)
$
be the set of candidate controls from which the planner chooses. The planner evaluates candidates through open-loop recursion. Define the planner's compensated state belief  $\bar x_{\ell+1|c} = f(\bar x_{\ell|c},u_{\ell|c})$ and corresponding correction-probabilities evolution as $q_{\ell|c} = \mathcal Q(\bar x_{\ell|c}, r_{\ell|c}, I_c)$. Through bias-prediction, the first and second moments propagate according to $(m_{\ell+1|c},P_{\ell+1|c}) = \mathcal B(m_{\ell|c},P_{\ell|c},u_{\ell|c},q_{\ell|c})$.

 Define the estimator-risk cost given the selected control as
\begin{equation}
    J_E(\mathbf u_c)
    = \sum_{k=c}^N\mathcal R_E(m_{k|c},P_{k|c})
\end{equation}
Candidate controls affect $J_E$ through both error flow and $q$

We define the optimal candidate selection as 
\begin{equation}
    \mathbf u_c^\star
    =
    \arg\min_{\mathbf u_c\in\mathbb U_c}
    \left[
        J_{\rm nom}(\mathbf u_c)
        +
        w_E J_E(\mathbf u_c)
    \right].
    \label{eq:planning_problem}
\end{equation}
where $J_{\rm nom}$ prices the error-compensated mean state. The first segment of the optimal policy $u_c=u_{c|c}^\star$ is applied, and the replanning closes the loop at the next estimator report.
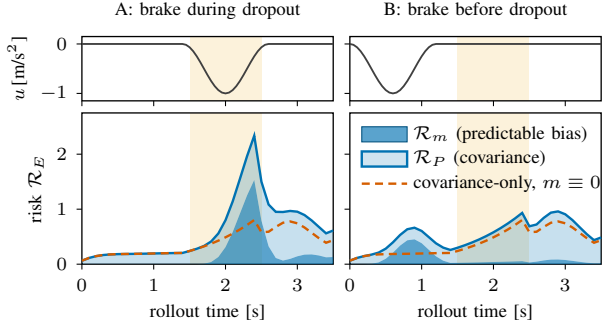
\begin{figure}[t]
    \centering
\begin{tikzpicture}
  \begin{groupplot}[bsp base,
      group style={group size=2 by 2, horizontal sep=0.22cm, vertical sep=0.16cm,
                   xticklabels at=edge bottom, xlabels at=edge bottom,
                   yticklabels at=edge left, ylabels at=edge left},
      width=0.385\columnwidth,
      xmin=0, xmax=3.5, xtick={0,1,2,3},
      xlabel={rollout time [s]}]
    \nextgroupplot[height=0.85cm, ymin=-1.15, ymax=0.15, ytick={-1,0},
                   ylabel={$u$\,[m/s$^2$]},
                   title={A: brake during dropout}]
      \bspshade{1.5}{2.5}{-1.15}{0.15}
      \addplot[black!75, line width=0.7pt] table[x=t, y=u] {\bspdata/planning_A.dat};
    \nextgroupplot[height=0.85cm, ymin=-1.15, ymax=0.15, ytick={-1,0},
                   title={B: brake before dropout}]
      \bspshade{1.5}{2.5}{-1.15}{0.15}
      \addplot[black!75, line width=0.7pt] table[x=t, y=u] {\bspdata/planning_B.dat};
    \nextgroupplot[height=2.0cm, ymin=0, ymax=2.75, ytick={0,1,2},
                   ylabel={risk $\mathcal R_E$}]
      \bspshade{1.5}{2.5}{0}{2.75}
      \addplot[name path=Rm, draw=none, forget plot] table[x=t, y=Rm, each nth point=2] {\bspdata/planning_A.dat};
      \addplot[name path=RE, bsp plan, forget plot] table[x=t, y=RE, each nth point=2] {\bspdata/planning_A.dat};
      \path[name path=ax] (axis cs:0,0) -- (axis cs:3.5,0);
      \addplot[draw=none, fill=bspPlan, fill opacity=0.65, forget plot] fill between[of=Rm and ax];
      \addplot[bsp planband] fill between[of=RE and Rm];
      \addplot[bsp infl, forget plot] table[x=t, y=Rzero, each nth point=2] {\bspdata/planning_A.dat};
    \nextgroupplot[height=2.0cm, ymin=0, ymax=2.75, ytick={0,1,2},
                   legend style={at={(0.03,0.97)}, anchor=north west, fill=none}]
      \bspshade{1.5}{2.5}{0}{2.75}
      \addplot[name path=Rm, draw=none, forget plot] table[x=t, y=Rm, each nth point=2] {\bspdata/planning_B.dat};
      \addplot[name path=RE, bsp plan, forget plot] table[x=t, y=RE, each nth point=2] {\bspdata/planning_B.dat};
      \path[name path=ax] (axis cs:0,0) -- (axis cs:3.5,0);
      \addplot[draw=none, fill=bspPlan, fill opacity=0.65, forget plot] fill between[of=Rm and ax];
      \addplot[draw=none, fill=bspPlan, fill opacity=0.22, forget plot] fill between[of=RE and Rm];
      \addlegendimage{area legend, draw=bspPlan!70!black, line width=0.4pt, fill=bspPlan, fill opacity=0.65}
      \addlegendentry{$\mathcal R_m$ (predictable bias)}
      \addlegendimage{area legend, draw=bspPlan, line width=0.9pt, fill=bspPlan!22}
      \addlegendentry{$\mathcal R_P$ (covariance)}
      \addplot[bsp infl] table[x=t, y=Rzero, each nth point=2] {\bspdata/planning_B.dat};
      \addlegendentry{covariance-only, $m\equiv0$}
  \end{groupplot}
\end{tikzpicture}
    \caption{Candidate commands can affect estimator error despite identical sensing (illustrative one-dimensional task-relative model). Both candidates lose observability over the same interval (shaded), so their correction probabilities $q$ are identical. Both candidates also command the same total deceleration, with candidate A braking during sensing loss and B braking before. A covariance-only evaluation ($m \equiv 0$, dashed), typical of Gaussian belief-space planning approaches, predicts the same risk for both trajectories. The full model attributes predictable estimator error to A's unobserved maneuvering, so A's risk, in tolerance-normalized units according to \eqref{eq:estimation_risk}, is dominated by $\mathcal R_m$ (dark), peaking at 2.5 $\times$ that of trajectory B.  }
    \label{fig:planning}
\end{figure}

\subsection{Cross-entropy optimization}
\label{sec:cem_planning}

We address this planning problem using a Cross-Entropy-based receding-horizon formulation \cite{gonzales, marin}.

When the estimator outputs information $I_k$, error moments $(m_c, P_c)$ are initialized, and the planner compensates the estimates according to \eqref{eq:compensation}. Candidate command sequences are then sampled from a distribution. For each command tape, physical nominal dynamics and estimator-error moments are propagated, with which $\mathcal Q$, $G_k$, $M$, and $\mathcal R_\mathrm{E}$ are evaluated. The objective of \eqref{eq:planning_problem} is then evaluated, and the minimum-cost trajectory candidate is selected. The first command segment of this trajectory is applied, the candidate distribution is shifted for the next sampling pass, and the algorithm is repeated from the next public report.
\section{Identification and Validation}
\label{sec:identification}
A model may be constructed analytically, using known error dynamics of the system, or identified using data generated from source-trajectories. A fitted model must pass both development, selection, and closed-loop testing conditions to be considered a fully qualified error model.

\subsection{Initialization and mean-error identification}
\label{sec:mean_identification}

At update $k$, let $\nu_k$ be the estimator's public innovation. Initializing $m_c$ from a causal affine map of the most recent $L$ innovations available at planning time $c$, the coefficients $a_0$ and $A_i^\nu$ are fit offline on development data, to satisfy the equation:
\begin{equation}
     m_c
    =
    a_0
    +
    \sum_{i=0}^{L-1}
    A_i^\nu \nu_{c-i},
    \label{eq:initializer_fit}
\end{equation}

Define the between-correction rate-of-change of the estimator error

\begin{equation}
    \Delta_k
    \triangleq
    \frac{
        e_{k+1}^- - e_k^+
    }{\delta}.
    \label{eq:flow_target}
\end{equation}

as the regression target for parameter fitting, excluding correction jumps from the fit. These correction jumps are handled separately by $J$, $W$, and $S^c$.

The rate-model coefficients are fit to this target accordingly:

\begin{equation}
    \Delta_k
    =
    A e_k^+
    +
    b
    +
    B\phi_k
    +
    \Lambda_z z_k
    +
    \varepsilon_k^{\rm fit},
    \label{eq:flow_fit}
\end{equation}
which directly implies \eqref{eq:error_flow} under forward Euler. In particular, $Ae_k^+$ is the model dependence on current estimator error, $b$ is the autonomous drift coefficient, $B\phi_k$ is the dependence predictable from planner-available features (e.g. command-related and recency-related features), $\Lambda_zz_k$ is the temporally correlated residual structure, and $\epsilon_k$ is the remaining, unexplained fitting residual.

\subsection{Residual covariance}
\label{sec:uncertainty_identification}
After fitting the mean error dynamics model, we form the one-step residual $\widehat\eta_k
    =
    e_{k+1}^-
    -
    F e_k^+
    -
    \delta
    \left(
        b+B\phi_k+\Lambda_z z_k
    \right),
    $ by subtracting the model's predicted error increment from the observed between-correction transition.
 The recency covariance terms are fit to the covariance of this residual, modeled by \eqref{eq:flow_covariance}. The $Q^{(i)}$ terms are fit using between-correction transitions from development data, targeting the empirical second moments $\hat \eta_k \hat\eta_k^\top$ and regressors $1$, $(1-r_{j,k})$, and $(1-r_{j,k})^2$ subject to positive-semidefiniteness constraints.
\subsection{Correction and acceptance identification}
\label{sec:correction_identification}

For accepted-correction events, $J_k$ is taken directly from the estimator's public update map when available. Otherwise, it is identified from paired pre- and post- correction errors by weighted least squares $\widehat J
=
\arg\min_J
\sum_{k}
\left|
e_k^{\rm acc}-Je_k^-
\right|_2^2$. Using $\widehat\epsilon_k^c
=
e_k^{\rm acc}
-
\widehat J_k e_k^-,$ the correction-noise covariance and pre-correction cross-covariance are estimated from the residual second moments. These moments are estimated within the same operating regimes used by the selected error model.

The parameters of the correction-probability mapping $\mathcal Q$ are fit from the binary labels $a_k$ using the candidate-correction samples $\widehat\theta_q
=
\arg\min_{\theta_q}
-
\sum_{k}
[
a_k\log q_k(\theta_q)
+
(1-a_k)\log\left(1-q_k(\theta_q)\right)]$. Known cadence or geometric components of $\mathcal Q$ are supplied analytically. Only unidentified components are estimated from data.


\section{SHIPBOARD LANDING: MODEL IDENTIFICATION AND CLOSED-LOOP EVALUATION}
We instantiate the surrogate planning problem for landing a T1 Ranger
vertical-takeoff-and-landing vehicle on a ship deck.
The VTOL is represented by a 20-component state comprising position, unit
quaternion, translational velocity, angular velocity, four aerodynamic-surface
joint positions, and three thrust states. The seven control inputs comprise four
joint-rate commands and three throttle commands. The vehicle dynamics are
evaluated using an aerodynamic rigid-body model with actuator states. We then identify predictive error models for, and consider planner performance across, three independent estimators: a truth-driven estimator error-emulator, an inertial Extended Kalman Filter (EKF), and an inertial Gaussian-sum filter. 

Prior work on ship-deck recovery problems assume usable state estimates and address the control and visual servoing problem presented to the planner \cite{cho2022, herisse2012, falanga2017}. We seek to investigate what happens to the estimate when the deck target leaves the sensors' fields of regard, and how the planner should respond.

\subsection{Shipboard task, dynamics, sensing, and success}

Define \(W\), \(B\), and \(D\) as the world, VTOL-body, and deck frames,
respectively, and let \(R_{AB}\) map coordinates from frame \(B\) into frame
\(A\). The deck-relative state is defined by
\begin{equation}
\begin{aligned}
p
&=R_{DW}(p_B^W-p_D^W),
&
R
&=R_{DW}R_{WB},
\\
v
&=R_{DW}(v_B^W-v_D^W)-\omega_{WD}^{D}\times p,
&
\omega
&=R\omega_{WB}^{B}-\omega_{WD}^{D}.
\end{aligned}
\label{eq:deck_coords}
\end{equation}
Thus, \(v=\dot p\) includes the rotating-deck-frame contribution. We denote
the resulting relative state by \(\zeta_D=(p,R,v,\omega)\).

The aircraft state is
\(x_A=(p_B^W,q_{WB},v_B^W,\omega_{WB}^{B},\theta,T)\in\mathbb R^{20}\),
where \(\theta\in\mathbb R^4\) contains the joint positions and
\(T\in\mathbb R^3\) the thrust states. Its dynamics are
\(\dot x_A=f_{T1}(x_A,u^{\rm act})\), with four commanded joint rates and
three throttle inputs; the thrust states have time constant
\(\tau_T=0.040\,\mathrm{s}\).

The planner defines the task-relative coordinates using the deck-relative coordinates above, with
\begin{equation}
y=
\begin{bmatrix}
p-p_*\\
\operatorname{Log}(RR_*^\top)\\
v-v_*\\
\omega
\end{bmatrix},
\quad
R_*=I.
\label{eq:shipboard_task}
\end{equation}
The planner-facing estimator-error coordinates use the corresponding 12-dimensional local error $e = [\delta p^\top, \delta \phi^\top, \delta v^\top, \delta\omega^\top]^\top$ where $\delta\phi = \operatorname{Log}(R\hat R^\top)$.

Simulation terminates at first physical contact between the aircraft and the ship, or planner-declared abort. Success is then evaluated using the declared acceptable landing footprint and relative landing attitude, velocities, and angular rates.

\subsection{Evaluated estimators and planners}
We consider three estimator architectures: \textbf{E1}, a truth-driven controlled-error emulator, providing a biased truth-driven estimator to the planner, for mechanism validation, \textbf{E2}, an IMU-driven error-state EKF with aircraft/deck navigation, gyroscope and accelerometer biases, ignorant of commands, and \textbf{E3}, a two-mode Gaussian sum of the inertial EKFs, for which public covariance includes between-mode spread. These estimator backends differ in how they obtain their error response coefficients, with closed-form derivation for E1 and offline fitting for the filters. 
Anchors are initialized using a learned affine map of the four most recent causal 12D innovations, with separately fitted residual covariance $P_0$.

We additionally evaluate three planner variants: \textbf{P1}, our formulation, which compensates for predicted estimator error and penalizes an estimator risk metric, \textbf{P2}, a traditional belief-space planner for which covariance trace is the only penalized uncertainty term, and \textbf{P3}, a nominal-only planner. The planners are tuned identically, differing only in additional cost function terms.

\subsection{Predictive error model identification and evaluation}
\label{sec:predictive_results}

We first evaluate whether the estimator-error model predicts error unavailable to a zero-mean model, and then whether its predicted moments are calibrated against subsequently realized estimator error.

The identified contracts exhibit strongly horizon and operating regime-dependent behavior. We therefore report prediction error by forecast horizon and trajectory regime rather than relying only on pooled loss, and discuss one of the qualified E2 error models as an illustrative case. In this model, the task-scaled mean-prediction loss was 8.16\% lower than the zero-mean predictor when pooled over two experimental datasets. However, this improvement was not uniform with horizon. Over evaluation horizons of 0.25, 0.5, 1.0, and 2.0 s, the error reductions were 1.61\%, 7.61\%, 12.58\%, and 6.39\% respectively. Through the closed-loop experiments, E1 exhibited similarly horizon-dependent error dynamics to those of the identified E2 and E3 models. Along horizons of less than 0.25s, the E1 model reduced squared task-scaled prediction loss, relative to a zero-mean baseline, by 75.99\%, which dropped to a 17.98\% reduction between 0.25 and 0.5s, before an increase to 76.39\% reduction for horizons longer than 2s. 

Prediction quality also depended strongly on trajectory regime. In a separate E2 evaluation, the selected model reduced pooled prediction loss by 2.33\%, but this aggregate improvement was concentrated in an operating regime with large velocities and task-relative offsets, where loss decreased by 5.52\%. The same model performed worse than the zero-mean predictor during near and intermediate-range operating regimes, by 3.0\% and 1.01\% respectively. Hence, an error model that appears beneficial when averaged across a task distribution will not necessarily improve prediction in every portion of the maneuver. This regime dependence is particularly relevant in receding-horizon planning, because the distribution visited by the planner changes systematically as the vehicle progresses from approach to terminal alignment. These experiments show that estimator-error prediction is neither globally stationary nor characterized adequately by a single pooled accuracy or calibration metric. Predictive abilities can strengthen or reverse with horizon, transfer unevenly between regimes, and coexist with poor covariance calibration. We therefore interpret the planner-conditioned error model as a \emph{qualified local predictive model}, whose claimed validity is tied to the trajectory regimes, forecast horizons, and source populations on which its mean and covariance has been evaluated.

\begin{table}[t]
\centering
\scriptsize
\setlength{\tabcolsep}{3.2pt}
\renewcommand{\arraystretch}{1.05}
\begin{tabular}{llrrrr}
\toprule
Estimator & Candidate & $L_{\rm auto}/L_0$ & $L_{\rm cond}/L_0$ & $\Delta_\mu$ & $\Delta_{\rm cond}$ \\
\midrule
E2 & Best predictive & 1.001 & 0.580 & +42.0\% & +42.1\% \\
E2 & Qualified & 0.739 & 0.640 & +36.0\% & +13.4\% \\
E2 & Zero-mean baseline & 1.000 & 1.000 & +0.0\% & +0.0\% \\
\addlinespace[1pt]
E3 & Best predictive & 0.846 & 0.727 & +27.3\% & +14.0\% \\
E3 & Qualified & 1.000 & 1.000 & +0.0\% & +0.0\% \\
\bottomrule
\end{tabular}
\figstylecaption
\caption{Predictive performance of representative planner-conditioned error models. Losses are normalized by the zero-mean baseline $L_0$. $L_{\rm cond}$ is the loss of models that condition on planner information, while $L_{\rm auto}$ is the loss of those that ignore it. Positive $\Delta_\mu$ and $\Delta_{\rm cond}$ indicate improvement over the zero-mean and planner-ignorant baselines, respectively.}
\label{tab:contract_performance}
\end{table}

\subsection{Closed-loop landing performance}
\label{sec:closed_loop_results}
\begin{figure}[t]
    \centering
    \input{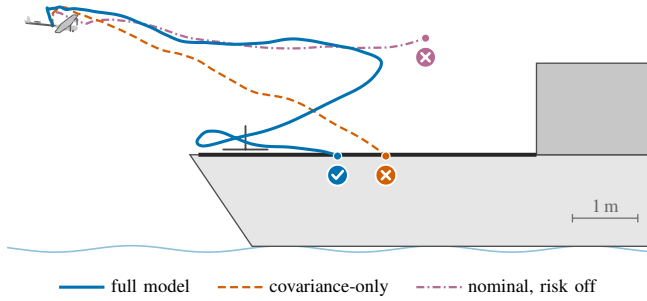}
    \caption{Representative 2D projection of matched ship-deck landing trial. Curves show measured deck-relative truth trajectories for P1 (blue), P2 (orange), and P3 (purple). P1 reaches the capture region $\kappa = 0.73$ while P2 ($\kappa = 1.14$) and P3 ($\kappa = 15.03$) terminate outside it. }
    \label{fig:landing}
\end{figure}

Across 16 matched trial groups using E1, P1 achieved 8/16 physical captures, compared with 6/16 for both P2 and P3. In the paired P1-P2 comparison, P1 was uniquely successful on five trials while P2 was uniquely successful on three. Notably, P1 obtained lower component-level capture gauges than P2 in pitch and pitch rate in 11/16 trials and in longitudinal velocity in 12/16 trials. Terminal task gauge also favored planner P1, with median gauges of 0.99, 1.60, and 1.87 for P1, P2, and P3 respectively. The corresponding mean gauges were 3.69, 9.71, and 3.57, respectively, although the P2 performance mean was strongly influenced by an isolated extreme failure. Thus, the matched trials show a modest observed capture advantage and improved typical terminal task state for P1, while the available cohort remains too small to establish a difference in capture probability. Beyond 2s, P2 underpredicted realized estimator error by a factor of 5.73, whereas P1 produced an observed-to-predicted ratio of 1.31.


\section{Conclusion}
This work introduced a model of planner-conditioned estimator error and its evolution through intermittent observation.

Our experiments show that this error can contain substantial predictable structure, though that structure is local, with prediction accuracy and risk calibration varying strongly with forecast horizon and operating regime. While our results show modest task-success and terminal-gauge improvement over existing planning formulations, our findings indicate that predictable estimator error can be planning-relevant, but exploiting it requires modeling identification over the regimes in which the model will be used. Future work should develop adaptive or regime-conditioned models, improve calibration of long-horizon risk, and evaluate  the approach over larger closed-loop cohorts and stochastic deck motion.

\label{sec:conclusion}

\section*{Acknowledgment}
ChatGPT 5.6 Sol (OpenAI) was used in the preparation of this document to generate TikZ code for figures \ref{fig:teaser}, \ref{fig:correction}, and \ref{fig:planning}, conduct preliminary literature reviews, format equations in \LaTeX, and proofread for notational consistency and grammar. All figures, equations, citations, and text were reviewed and confirmed by the authors, and the authors claim full responsibility for the content of this work.


\bibliography{Acc_bsp}

\bibliographystyle{IEEEtran}

\end{document}